\documentclass[twoside]{article}

\usepackage[accepted]{aistats2020}
\usepackage[utf8]{inputenc} % allow utf-8 input
\usepackage[T1]{fontenc}    % use 8-bit T1 fonts
\usepackage{hyperref}       % hyperlinks
\usepackage{url}            % simple URL typesetting
\usepackage{booktabs}       % professional-quality tables
\usepackage{amsfonts}       % blackboard math symbols
\usepackage{nicefrac}       % compact symbols for 1/2, etc.
\usepackage{microtype}      % microtypography
\usepackage{xspace}
\usepackage{graphicx}
\usepackage{subfigure}
\usepackage{algorithm}
\usepackage{algorithmic}
\usepackage{multirow}
\usepackage{wrapfig}
\usepackage{enumitem}

\usepackage{amsmath,amsfonts,amsthm,bm}
\usepackage{bbm}

\newcommand{\comment}[1]{}

\newtheorem{theorem}{Theorem}

\def\eqref#1{equation~\ref{#1}}
\def\1{\mathbbm{1}}

\DeclareMathAlphabet{\mathsfit}{\encodingdefault}{\sfdefault}{m}{sl}
\SetMathAlphabet{\mathsfit}{bold}{\encodingdefault}{\sfdefault}{bx}{n}

\def\gA{{\mathcal{A}}}

\def\gD{{\mathcal{D}}}

\def\gF{{\mathcal{F}}}

\def\gX{{\mathcal{X}}}
\def\gY{{\mathcal{Y}}}

\newcommand{\R}{\mathbb{R}}

\DeclareMathOperator*{\argmin}{arg\,min}

\def\deriv{\partial}
\def\X{\gX}
\def\Y{\gY}
\def\A{\gA}
\def\D{\gD}

\def\lossfn{\mathcal{L}}
\def\dataset{\mathcal{D}}
\def\scorefn{\mathcal{R}}
\def\biasfn{\mathcal{B}}
\def\slack{\beta}
\def\maxslack{\slack_{\mathrm{max}}}
\def\PP{\mathrm{PP}}
\def\opt{\mathrm{OPT}}
\def\tpr{\mathrm{TPR}}
\def\fpr{\mathrm{FPR}}
\def\pr{\mathrm{PR}}

\newtheorem{definition}{Definition}

\theoremstyle{definition}

\usepackage{xcolor}

\begin{document}


\runningauthor{Nachum and Jiang}

\twocolumn[

\aistatstitle{Group-based Fair Learning Leads to Counter-intuitive Predictions}

\aistatsauthor{Ofir Nachum$^*$ \And Heinrich Jiang$^*$}

\aistatsaddress{Google Research \And Google Research} ]

\begin{abstract}
A number of machine learning (ML) methods have been proposed recently to maximize model predictive accuracy while enforcing notions of group parity or {\em fairness} across sub-populations.
We propose a desirable property for these procedures, {\em slack-consistency}: For any individual, the predictions of the model should be monotonic with respect to allowed {\em slack} (i.e., maximum allowed group-parity violation).
Such monotonicity can be useful for individuals to understand the impact of enforcing fairness on their predictions.
Surprisingly, we find that standard ML methods for enforcing fairness violate this basic property.
Moreover, this undesirable behavior arises in situations agnostic to the complexity of the underlying model or approximate optimizations, suggesting that the simple act of incorporating a constraint can lead to drastically unintended behavior in ML.
We present a simple theoretical method for enforcing slack-consistency, while encouraging further discussions on the unintended behaviors potentially induced when enforcing group-based parity.
\end{abstract}
\section{INTRODUCTION}

Algorithmic fairness in machine learning (ML) has recently become an important concern. Without appropriate intervention during the pre-processing, training, or inference stages of an ML procedure, the resulting model can be biased against certain groups~\cite{angwin,hardt2016equality,jiang2019identifying}. Accordingly, enforcing group-based fairness notions such as demographic parity \cite{dwork2012fairness} and equal opportunity \cite{hardt2016equality} is a difficult and active research problem. %: biases are often embedded in the data in possibly complex ways, so simple approaches such as ignoring the features corresponding to the protected groups are largely ineffective \cite{pedreshi2008discrimination}.
%Training fair classifiers has received a great deal of attention and a number of approaches have been proposed. 

One common approach to enforcing fairness is {\em post-processing}. In post-processing, one first learns a score function without concerns for fairness, and then decision thresholds are chosen for each group to ensure that various notions of group parity are satisfied~\cite{doherty2012information,feldman2015computational,hardt2016equality}. 
Other than post-processing, {\em constrained optimization} provides another common approach, in which enforcing group parity is framed as a constraint on a minimum loss objective. The resulting constrained optimization problem is then solved through the use of Lagrange multipliers~\cite{zafar2015fairness,goh2016satisfying,eban2017scalable,cotter2018two,cotter2018optimization,agarwal2018reductions,kearns2017preventing}.

%While group-based fairness notions may provide guarantees for the group on average, there are a number of disadvantages. For example, it is known that there is a gap between enforcing group-based notions of fairness and individual notions of fairness \cite{dwork2012fairness,joseph2016fairness,kim2018fairness}. 

In practice, a model is usually not required to perfectly satisfy the fairness constraint, but rather some amount of \emph{slack} is allowed to trade-off an allowable degree of bias with better accuracy~\cite{zafar2015fairness}. In this paper, we explore the behavior of standard machine learning fairness procedures as this amount of allowable slack changes. We propose a desirable property, \emph{slack-consistency}, which expects that fair learning procedures should give predictions that are monotonic in the amount of slack allowed. Such slack-consistency is intuitive: one would expect that for any given individual, there would be a prediction under no slack (i.e. perfectly satisfy fairness constraint) and a prediction under infinite slack (i.e. unconstrained), and that for any slack in between, the predictions would change monotonically between these two extremes. Moreover, given this behavior, we can move towards \emph{explainable} ML fairness, where individuals can understand the impact of their predictions based on the amount of enforced fairness, thus making the process more transparent.

We show that, surprisingly, popular group-based fairness methods fail to satisfy slack-consistency in both real-world and simple synthetic datasets, and this failure arises in situations agnostic to the complexity of the underlying model.
Even for the post-processing method, which may be implemented agnostic to optimization errors (using an exhaustive search over thresholds), slack-inconsistency can arise in a variety of settings.
%Moreover, we observe that slack-inconsistency occurs not only for individuals but also for the group on average.
Our findings thus show that the consequences of imposing group-based fairness notions on ML models are poorly understood, and these consequences are often at odds with intuitive beliefs of what fairness should encourage an ML model to do.
We propose a simple ML fairness method which provably possesses slack-consistency but at the same time encourage further discussions on the utility of imposing group-based fairness notions in general.

%\begin{itemize}[topsep=2pt,itemsep=2pt,partopsep=2pt, parsep=2pt]
%    \item Section~\ref{sec:background} introduces notation.
%    \item Section~\ref{sec:slack_consistency} defines slack-consistency.
%    \item Section~\ref{sec:popular_methods} shows counter-examples where the popular group-based %fairness methods fail to satisfy slack-consistency.
%    \item Section~\ref{sec:gabos} presents a theoretical method that satisfies slack-consistency given a Bayes-optimal score function and randomized classifiers. 
%\end{itemize}
\section{BACKGROUND}\label{sec:background}
We consider a fair machine learning setting, in which individuals correspond to pairs $(x,a)$ where $x\in\X$ is a vector of {\em features} associated with the individual and $a\in\A$ is an additional feature corresponding to group membership.  At times we will write $\A(x)$ as a function which returns the group membership of $x$.  
For simplicity, we assume two groups; i.e, $\A=\{1, 2\}$.  We are given some dataset $\D = \{(x_i,a_i,y_i)\}_{i=1}^N$ where $y\in\Y$ is an observed {\em label}.  For simplicity, we consider the binary classification setting; i.e, $\Y=\{0,1\}$. 

In standard machine learning, one is tasked with finding some (potentially stochastic) classifier $f:\X\times\A\to[0,1]$ within some family $\gF$ which minimizes a loss function $\lossfn(f,\D)$ on the dataset (e.g., mis-classification rate).  When group-based fairness notions are imposed, the task is modified to finding an optimal loss classifier $f$ within the subset of {\em unbiased} functions of $\gF$; i.e., $\biasfn(f, \D)=0$, where $\biasfn$ measures the bias of $f$ on $\D$.

Many works in the literature express the bias function $\biasfn$ in terms of some disparity of the predictions of $f$ between the two groups.
For example, {\em demographic parity}~\cite{dwork2012fairness} measures the unfairness of $f$ as the difference in positive prediction rates:
\begin{align*}
    \biasfn_{\mathrm{DemPar}}(f, \D) := &\frac{\sum_{i=1}^N f(x_i, a_i)~\1[a_i=1]}{\sum_{i=1}^N \1[a_i=1]} \\
    &- \frac{\sum_{i=1}^N f(x_i, a_i)~\1[a_i=2]}{\sum_{i=1}^N \1[a_i=2]}.
\end{align*}
Demographic parity, although simple, has been criticized for unnecessarily encouraging poor classifiers in pursuit of fairness~\cite{hardt2016equality,dwork2012fairness,kleinberg2016inherent}.  Accordingly, some works propose to define bias in terms of {\em equal opportunity}~\cite{hardt2016equality}, which measures disparity in true positive prediction rates:
\begin{align*}
    \biasfn_{\mathrm{EqOpp}}(f, \D) := &\frac{\sum_{i=1}^N f(x_i, a_i)~\1[a_i=1, y_i=1]}{\sum_{i=1}^N \1[a_i=1, y_i=1]} \\
    &- \frac{\sum_{i=1}^N f(x_i, a_i)~\1[a_i=2, y_i=1]}{\sum_{i=1}^N \1[a_i=2, y_i=1]}.
\end{align*}
Unlike demographic parity, the true classifier $f(x_i)=y_i$ always satisfies equal opportunity.

Additional variants of these constraints exist in the literature.  For example, {\em disparate impact}~\cite{feldman2015certifying} which enforces demographic parity while restricting the classifier $f$ to not use the protected attribute $a_i$ in its prediction $f(x_i,a_i)$.  The notion of {\em equal odds}~\cite{hardt2016equality} augments equal opportunity to enforce both equal true positive rates and equal false positive rates.
For simplicity, we will restrict our focus in this paper to the notions of demographic parity and equal opportunity, although our discussions and conclusions may easily extend to these more complicated notions of fairness.

\section{SLACK-CONSISTENCY}\label{sec:slack_consistency}
In many instances of fair machine learning, a model is not enforced to be perfectly fair.  Rather, it is enforced to be fair within some slack; i.e., it is allowed to have some bias, but that bias must be bounded $\biasfn(f,\D) < \maxslack$ or bounded absolutely $|\biasfn(f,\D)|<\maxslack$ for some $\maxslack$.
The reasons for this are two-fold: First, a perfectly fair, zero-bias classifier may not be feasible (or desirable) in practice. Second, the initial motivation for learning a fair classifier is often expressed in terms of some allowed amount of bias. For example, legal definitions of fairness often invoke the $p\%$-rule (commonly, the $80\%$-rule): The ratio between positively predicted individuals in one group versus another should not exceed $p\%$
~\cite{zafar2015fairness}.

Since any ML fairness optimization is affected by the allowed slack $\maxslack$ and this slack is specified by the problem formulation, it may be important to understand the precise ways in which a choice of $\maxslack$ affects the final classifier.  We propose the following property: 
\begin{definition}[Slack-consistency]
A procedure for learning fair classifiers with respect to some amount of allowed slack $\maxslack$ and a fixed dataset $\D$ is slack-consistent if, for any individual, the predictions of the learned classifier for this individual are monotonic with respect to $\maxslack$. 
\end{definition}
More precisely, let us denote the procedure for learning fair classifiers as $\opt$; i.e., the result $f_{\maxslack}:=\opt(\dataset, \maxslack)$ of running $\opt$ is a model that takes in features $x$ and group attribute $a$ of some individual and returns a probability $f_{\maxslack}(x,a)\in[0,1]$ of classifying this individual positively.
Then, $\opt$ is slack-consistent if the predictions $f_{\maxslack}(x, a)$ of any individual $(x,a)$ are monotonic with respect to $\maxslack$.  That is, for $\beta_1<\beta_2<\beta_3$, we must have either $f_{\beta_1}(x, a)\le f_{\beta_2}(x, a)\le f_{\beta_3}(x, a)$ or $f_{\beta_1}(x, a)\ge f_{\beta_2}(x, a)\ge f_{\beta_3}(x, a)$.

We argue that slack-consistency is a reasonable, intuitive, and desirable property for machine learning methods.  For example, consider an individual from a disadvantaged group who would be positively labeled with no fairness enforcement (unbounded slack $\maxslack\to\infty$).  Slack-consistency ensures that when fairness is enforced ($\maxslack\to 0$), the same individual should not be assigned a negative prediction (assuming a positive prediction at maximal favoring of the group $\maxslack\to-\infty$).  Otherwise, we would be unfairly disadvantaging the individual in the process of attempting to undo a disadvantage for the individual's group.
The slack-inconsistency of a model in this case may be interpreted as an implementation of a {\em self-fulfilling prophecy} (see~\cite{dwork2012fairness}; ``Catalog of Evils''); i.e., a vendor maliciously chooses the `wrong' members of a protected group to predict positively, ensuring that a future analysis will find that membership in the protected group is associated with less likelihood of positive outcomes.  In the converse setting, an individual from the advantaged group who would be negatively labeled with no fairness enforcement should not be positively labeled when fairness is enforced. % (this can be interpreted as an instance of {\em reverse tokenism} from the ``Catalog of Evils'').

The rest of the paper is organized as follows.
In Section~\ref{sec:popular_methods} we will investigate the behavior of popular methods for learning fair ML models in terms of slack-consistency, starting with constrained optimization and then focusing on post-processing.  Surprisingly, we will find that these methods fail to satisfy slack-consistency in almost all settings.
Then, in Section~\ref{sec:gabos} we present a simple theoretical method that satisfies slack-consistency given a Bayes-optimal score function and randomized classifiers, although we concede that practical scenarios often do not permit such classifiers.

\section{POPULAR METHODS FAIL TO SATISFY SLACK-CONSISTENCY}\label{sec:popular_methods}
For simplicity, we consider the absolutely bounded bias setting, in which the bias of $f$ is restricted to $|\biasfn(f,\D)|\le\maxslack$ for $\maxslack>0$, although our findings can be generalized.  We will show that the two common methods for learning fair classifiers, {\em constrained optimization} via the method of Lagrange multipliers and {\em post-processing} via exhaustive threshold search, often fail to satisfy slack consistency. 

For our experiments on non-synthetic data, we employ the following datasets:

\begin{itemize}[leftmargin=*,noitemsep,topsep=0pt,parsep=2pt,partopsep=0pt]
\item {\bf Adult} \cite{lichman2013uci} ($48842$
examples). The task is to predict whether the person's income is more than $50$k per year based on census data. We use $2$ protected groups based on gender and preprocess the dataset consistent with previous works, e.g. \cite{zafar2015fairness,goh2016satisfying}.
\item {\bf Communities and Crime} \cite{lichman2013uci} ($1994$ examples). The task is to predict whether a community has high or low crime rate. We preprocessed the data consistent with previous works, e.g. \cite{cotter2018training} and form two protected groups based on whether the community's black population is above the median.
\item {\bf ProPublica’s COMPAS} recidivism data ($7,918$ examples).
The task is to predict recidivism based on 
criminal history, jail and prison time, demographics,
and risk scores. We preprocess this dataset in a similar way as the Adult dataset and use two gender-based  protected groups.
\end{itemize}

\subsection{Constrained Optimization}
In the constrained optimization approach, we have a loss function $\ell(\theta)$ and a fairness constraint $g(\theta) \le 0$ over parameter space $\theta \in \Theta$.
The Lagrangian is $\mathcal{L}(\theta, \lambda) := \ell(\theta) + \lambda \cdot g(\theta)$ where $\lambda \ge 0$ and the goal is to find a solution to $\min_{\theta \in \Theta} \max_{\lambda \ge 0} \mathcal{L}(\theta, \lambda)$. In fairness problems, the loss function is taken to be the usual loss function for a model (e.g. hinge loss) and the fairness constraints (possibly with slack) are typically relaxed so that they are differentiable (e.g. hinge relaxation) and we can alternatively apply SGD to minimize $\mathcal{L}$ in $\theta$ and maximize  $\mathcal{L}$ in $\lambda$ until convergence. This is the approach a number of works adopt \cite{zafar2015fairness,eban2017scalable,goh2016satisfying,cotter2018two}.

In general settings where $\ell$ is the loss of non-convex model such as a neural network,
%or the fairness constraints are incompatible with each other, 
it may not be surprising that properties such as slack-consistency could fail to hold. In the non-convex setting, as pointed out in recent works such as \cite{agarwal2018reductions,cotter2018two}, a saddle point to the Lagrangian may not even exist and thus models may have nothing to converge to. Moreover, even with convergence, there can be multiple solutions with different accuracy and fairness violations which nonetheless attain the same Lagrangian value. Similar behavior can happen with multiple fairness constraints which are in conflict, such as is known for equalized odds due to feasibility issues \cite{chouldechova2017fair,woodworth2017learning,kleinberg2016inherent}.

In this section, we consider the simplest of cases, where we use a linear model and a single fairness constraint (demographic parity or equal opportunity) on just two protected groups.  In this case, an optimal saddle point to the Lagrangian exists and joint SGD is guaranteed to converge to it~\cite{cotter2018two}. Despite these restrictions, we find that constrained optimization can still fail to satisfy slack-consistency. We illustrate this in Figure~\ref{fig:constrained_optimization} on a number of benchmark fairness datasets: Adult, COMPAS, and Communities and Crime. We train a linear model subject to hinge relaxations of the fairness constraints and jointly train the Lagrangian using the ADAM optimizer under default settings with minibatch size of $100$ for $20$ epochs. We then sort the solutions by the actual fairness violations in training (rather than the violations on the hinge relaxation) to account for the variability between the original and relaxed constraints. 

\begin{figure*}[h]
\begin{center}
\begin{tabular}{lll}
\includegraphics[width=0.32\textwidth]{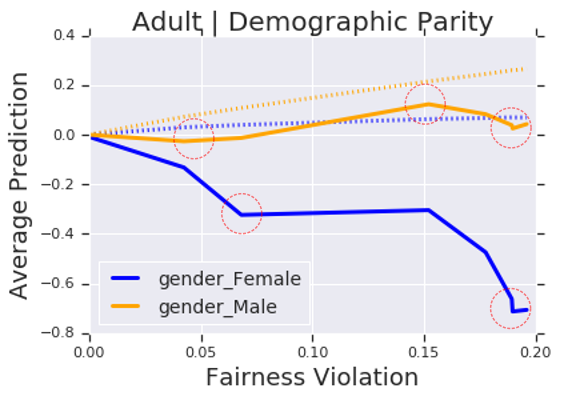}  & 
\includegraphics[width=0.32\textwidth]{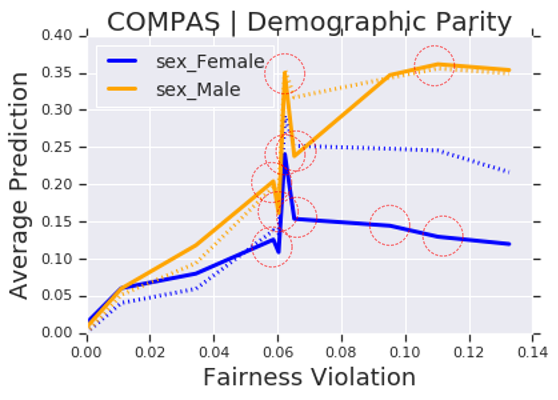}  & 
\includegraphics[width=0.32\textwidth]{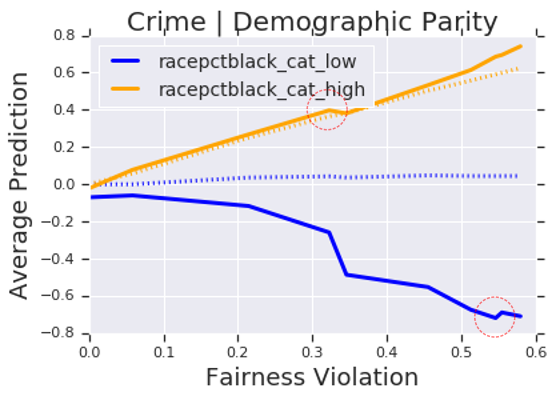} \\ 
\includegraphics[width=0.32\textwidth]{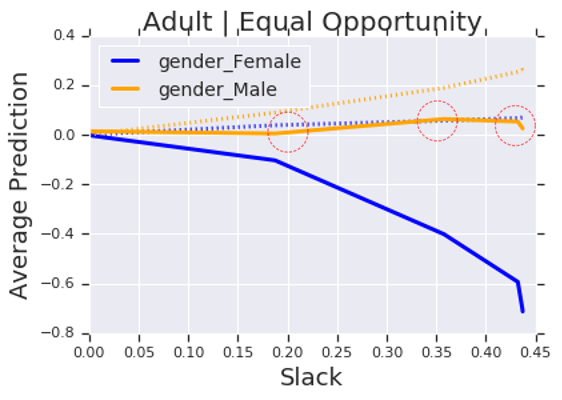} & 
\includegraphics[width=0.32\textwidth]{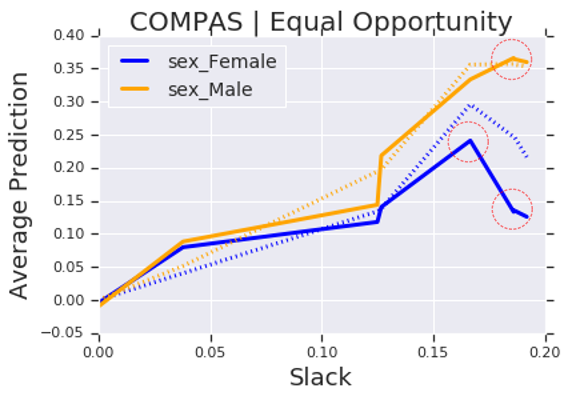}   & 
\includegraphics[width=0.32\textwidth]{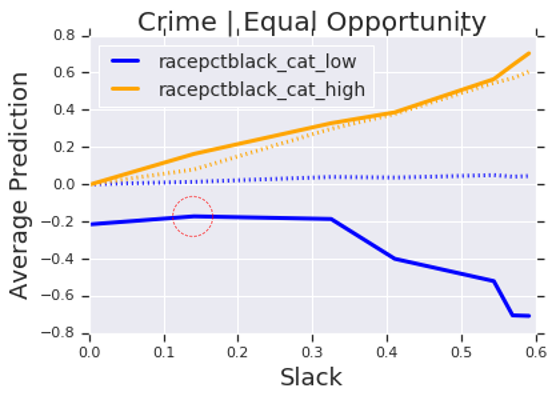} 
\end{tabular}
\end{center}
\caption{{\bf Constrained Optimization. Average model predictions for each protected group vs. slack}: {\bf Top}: Demographic Parity. {\bf Bottom}: Equal Opportunity. For each dataset, we train a linear classifier to satisfy fairness constraints over various slacks. We sort the solutions based on the amount of fairness violation on the training set and then plot the average prediction for each protected group within the training set. We plot both the average soft prediction score (solid lines) as well as the average {\it thresholded} hard binary prediction (dotted lines) for each group.   We see that in most cases, slack-consistency is violated; i.e., the average predictions scores are not monotonic with slack. Points on the curve that violate slack-consistency are circled in {\color{red} red}. It is also worth noting that with the constrained optimization approach, there is the additional counter-intuitive effect where the average thresholded prediction can increase while the average soft prediction decreases (see gender\_Female predictions for Adult) which is possible depending on the distribution of the features. }
  \label{fig:constrained_optimization}
\end{figure*}

Figure~\ref{fig:constrained_optimization} gives us an unsettling realization. It shows that the predictions for each protected group do not satisfy slack-consistency on average, which means that not only are there individuals whose predictions do not satisfy slack-consistency, but also that this property is not even maintained at the group level. 
It is also worth noting that the constrained optimization approach presents another counter-intuitive side-effect (also seen in Figure~\ref{fig:constrained_optimization}), in that the average soft prediction can increase (resp. decrease) while the average thresholded hard prediction decreases (resp. increases). 

Overall, we find that the counter-intuitive side effects associated with slack-inconsistency can easily arise for constrained optimization, even in the simple case of a linear model.
%Since, the constrained optimization approach optimizes for accuracy subject to fairness constraints at a fixed threshold and thus provides little guarantees at other thresholds. Moreover, since there are no guarantees about the monotonicity of the linear regression coefficients learned across different fairness slacks, such counter-intuitive side effects can arise depending on the distribution of the features.

\subsection{Post-Processing}

The post-processing method \cite{hardt2016equality} is perhaps one of the simplest approaches to fair classification. It starts with a score function $\scorefn(x,a)$ which maps individuals to a continuous value and then selects thresholds for each protected group so that the resulting binary classifier from these thresholds has minimum cost while satisfying the fairness constraints.  We provide a pseudocode of a slack-enabled version of post-processing in Algorithm~\ref{alg:post-processing}. Note that we utilize an exhaustive search to find the optimal thresholds. Thus, unlike in the constrained optimization setting, our results are agnostic to any approximations in the optimization.

In general, the post-processing method may require randomized thresholds (equivalent to the quantiles used by~\cite{sacchetto2018proper}).  For our discussions, we will express this through the use of {\em normalized thresholds}.  A normalized threshold $\tau\in[0,1]$ corresponds to a distribution over at most two (adjacent) thresholds which achieves a positive prediction rate of $1 - \tau$.  
Although normalized thresholds are required in general, we note that their use is not ideal.  In practical scenarios, a stochastic classifier can be seen as either capricious (if only one random classification is allowed per individual) or exploitable (if multiple random classifications are allowed).
We will write the loss $\lossfn$ and bias $\biasfn$ as functions of these normalized thresholds. We note that slack-consistency of post-processing is equivalent to monotonicity of the chosen normalized thresholds with respect to $\maxslack$.

\begin{algorithm}[h]
   \caption{Post-processing with tie-breaking.}
\begin{algorithmic}    \label{alg:post-processing}
   \STATE {\bf Inputs}: Allowed slack $\maxslack\in\R_{>0}$, functions $\lossfn,\biasfn$ that take normalized thresholds $\tau_1,\tau_2$ for each group and return the loss and bias, respectively, associated with these thresholds. %, with respect to some score function $\scorefn$.
   \STATE
   \STATE \textbullet~ Find all normalized thresholds $\tau_1,\tau_2\in[0,1]$ that satisfy,
   \begin{align*}
       (\tau_1,\tau_2) \in \argmin_{\hat{\tau}_1,\hat{\tau}_2} ~& \lossfn(\hat{\tau}_1, \hat{\tau}_2) 
      \hspace{0.3cm} \text{s.t.}  \hspace{0.3cm} |\biasfn(\hat{\tau}_1,\hat{\tau}_2)| \le \maxslack.
   \end{align*}
   \STATE \textbullet~ Find those solutions $(\tau_1,\tau_2)$ with lowest bias $|\biasfn(\tau_1,\tau_2)|$.  
   \STATE \textbullet~ Of those solutions with lowest bias, find those solutions with lowest threshold for the first group $\tau_1$.
   \STATE \textbullet~ Of those solutions with lowest bias and lowest $\tau_1$, return the solution with lowest threshold for the second group $\tau_2$.
\end{algorithmic}
\end{algorithm}

We present several counterexamples, which show that post-processing does not yield slack-consistency.
We begin by considering the application of post-processing to a score function which is not Bayes-optimal.
This scenario is typical in practice, where the score function is usually some learned function (e.g., a neural network or a decision tree ensemble).
Theorem~\ref{theorem:counterexample_demparity} provides an example of a dataset and such a score function for which post-processing yields slack-inconsistent solutions. 
%a continuous distribution where the score function is not Bayes optimal and the post-processing method fails to be slack-consistent. 

\begin{theorem}[Slack-inconsistency of post-processing on non-Bayes-optimal score function.]
\label{theorem:counterexample_demparity} There exists a distribution and score function such that the post-processing method fails to satisfy slack-consistency. 
\end{theorem}
\begin{proof}
Consider the distribution partitioned into two protected groups, $A = 1$ and $A = 2$ each occurring with equal proportion and let our score function be $R$ and suppose that we are in the binary classification setting with the goal of demographic parity. 
Let the distribution for $A = 1$ be as follows:
\begin{itemize}[leftmargin=*,noitemsep,topsep=0pt,parsep=2pt,partopsep=0pt]
    \item $\frac{1}{2}$ of the points have $R$ uniformly distributed in $[0, 0.5]$ and label $y = 1$ with probability $0.6$ %$0.7-0.4\cdot R$ 
    and label $y=0$ otherwise,
    \item $\frac{1}{4}$ of the points have $R$ uniformly distributed in $[0.5, 0.75]$ and label $y = 0$, 
    \item $\frac{1}{4}$ of the points have $R$ uniformly distributed in $[0.75, 1]$ and label $y = 1$.
\end{itemize}
Let the distribution of $A=2$ be as follows:
\begin{itemize}[leftmargin=*,noitemsep,topsep=0pt,parsep=2pt,partopsep=0pt]
    \item $\frac{1}{4}$ of the points have $R$ uniformly distributed in $[0, 0.25]$ and label $y = 0$,
    \item $\frac{1}{4}$ of the points have $R$ uniformly distributed in $[0.25, 0.5]$ and label $y = 1$, 
    \item $\frac{1}{4}$ of the points have $R$ uniformly distributed in $[0.5, 0.75]$ and label $y = 1$, 
    \item $\frac{1}{4}$ of the points have $R$ uniformly distributed in $[0.75, 1]$ and label $y = 0$.
\end{itemize}

We plot the misclassification rate with respect to chosen threshold for each group in Figure~\ref{fig:discrete_equal_opp_post_processing} (left).
Note that at any threshold $\tau$, the classifier $\1[R(x)\ge\tau]$ achieves positive prediction rate of $1-\tau$ on either group.  For strict fairness constraints ($\maxslack\to0$), the ideal thresholds are thus $\tau_1=\tau_2=0.25$.  If we choose to increase the allowed slack by some small amount to $\maxslack=\epsilon>0$, the ideal threshold for the first group will decrease, since the misclassification error has a positive derivative for the first group at $\tau_1=0.25$.  However, for a large enough slack, the ideal threshold for the first group will be at the global minimum, $\tau_1=0.75$.  Therefore, post-processing applied to this example yields slack-inconsistent solutions.
\end{proof}

In the previous theorem's counter-example, the score-function was not Bayes-optimal. We next consider a scenario for which we have a Bayes-optimal classifier but are not allowed to employ stochastic classifiers.
Indeed, stochastic classifiers are often undesirable in practice, since they can be seen as capricious (why should a random number determine my loan eligibility?) or exploitable (if I get denied a loan on my first try, I will apply again until I get accepted). 
We show that post-processing in this scenario fails to satisfy slack-consistency.
\begin{theorem}[Slack-inconsistency of post-processing on Bayes-optimal score function with deterministic thresholds.]\label{theorem:counterexample_bayesoptimal}
There exists a distribution where the score function is Bayes-optimal for each protected group and the post-processing method fails to satisfy slack-consistency.
\end{theorem}
\begin{proof}
See Figure~\ref{fig:discrete_equal_opp_post_processing}, right three images.
\end{proof}
%\begin{wrapfigure}{l}{0.5\textwidth}
%\begin{figure}[h]
%\begin{center}
%\begin{tabular}{c}
%\includegraphics[width=0.35\textwidth]{figures/pp_counter}
%\end{tabular}
%\end{center}
%\caption{Figure for Theorem~\ref{theorem:counterexample_demparity}.
%The plot shows cost (y-axis) with respect to chosen threshold (x-axis) for the first group (blue), the second group (orange), and their average (green).}
%  \label{fig:pp-counter}
%\end{figure}
%\end{wrapfigure}

\begin{figure*}[h]
\begin{center}
\begin{tabular}{llll}
\includegraphics[width=0.23\textwidth]{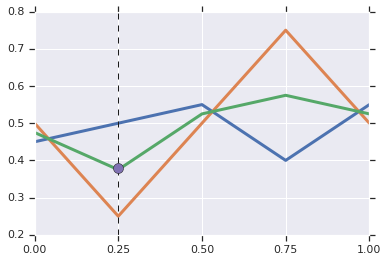} &
\includegraphics[width=0.23\textwidth]{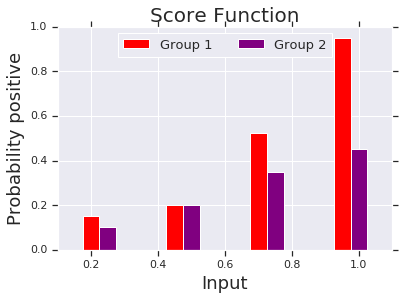}  & 
\includegraphics[width=0.23\textwidth]{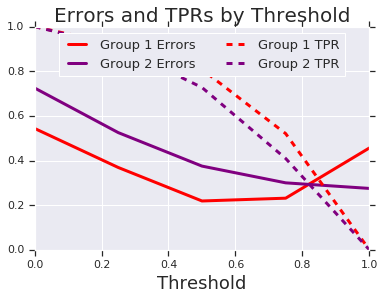}  & 
\includegraphics[width=0.23\textwidth]{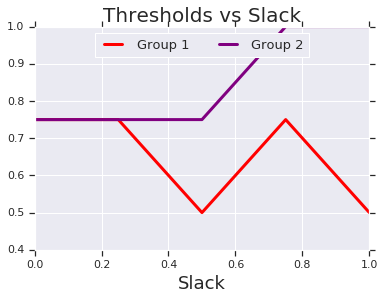} 
\end{tabular}
\end{center}
\caption{{\bf Left}: Figure for Theorem~\ref{theorem:counterexample_bayesoptimal}.
The plot shows cost (y-axis) with respect to chosen threshold (x-axis) for the first group (blue), the second group (orange), and their average (green). {\bf Right three images}: We show an example where the post-processing method has access to the Bayes-optimal score function but fails to satisfy slack-consistency under the equal opportunity fairness constraint. {\bf Middle-Left}: The data for each group takes on $4$ discrete values: $\frac{1}{4}, \frac{1}{2}, \frac{3}{4}, 1$, and plotted are the probabilities that the point is positively labeled for each group. For the first group, they are  $0.15, 0.2,  0.525,  0.95$, and for the second group, they are $0.1, 0.2, 0.35, 0.45$. {\bf Middle-Right}: Shown are the classification errors and true positive rates depending on where we set the thresholds for each group. {\bf Right}: The optimal thresholds found by the post-processing methods vs slack. These can be obtained over an exhaustive search over the $25$ possible classifiers obtained from all pairs of choices for thresholds. See the Appendix for a much more detailed display of the calculations to arrive at these conclusions.} 
  \label{fig:discrete_equal_opp_post_processing}
\end{figure*}

\begin{figure*}[h]
\begin{center}
\begin{tabular}{lll}
\includegraphics[width=0.32\textwidth]{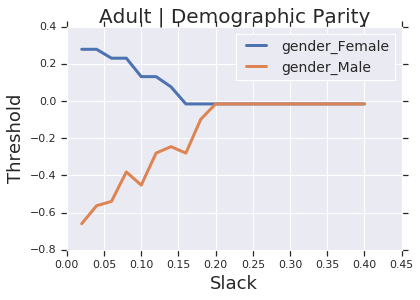}  & 
\includegraphics[width=0.32\textwidth]{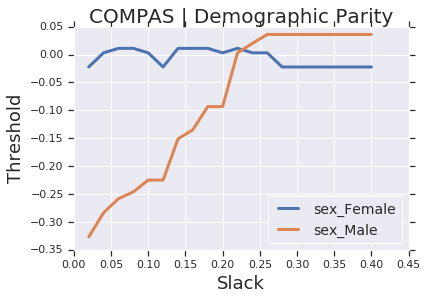}  & 
\includegraphics[width=0.32\textwidth]{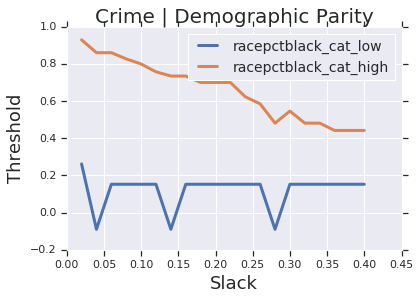} \\ 
\includegraphics[width=0.32\textwidth]{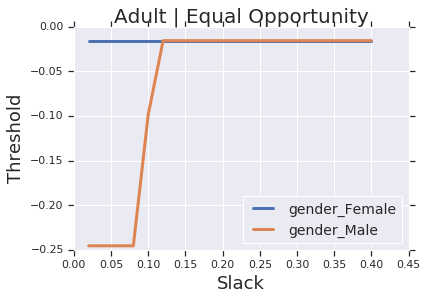} & 
\includegraphics[width=0.32\textwidth]{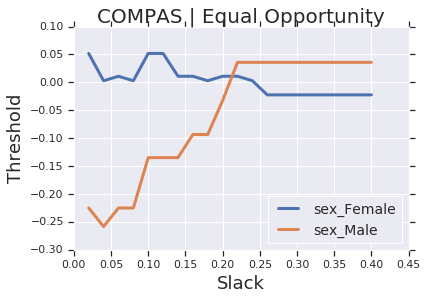}   & 
\includegraphics[width=0.32\textwidth]{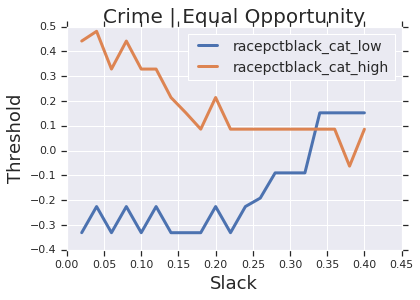} 
\end{tabular}
\end{center}
\caption{{\bf Post-Processing. Chosen thresholds for each protected group vs slack}: For each dataset, we run logistic regression to learn a score function and then apply Algorithm~\ref{alg:post-processing} across a range of slacks, plotting the chosen thresholds for each protected group. The top row shows the results for demographic parity and the bottom row shows the results for equal opportunity. We see that in most of the cases, the resulting thresholds do not satisfy slack-consistency.}
  \label{fig:postprocessing}
\end{figure*}

We conclude our counterexamples for post-processing with Figure~\ref{fig:postprocessing}, which shows that on real datasets, the post-processing method fails to be slack-consistent.

\section{POST-PROCESSING WITH GABOS FUNCTIONS}\label{sec:gabos}
While the previous section showed that post-processing yields slack-inconsistent solutions when one lacks either a Bayes-optimal classifier or access to stochastic thresholds, in this section we show that post-processing is guaranteed to be slack-consistent given these two assumptions.

We begin by defining a Bayes-optimal score function.  Notably, the function must be Bayes-optimal with respect to both features and group membership attribute (rather than with respect to only the features of the individual, which is how ML models are typically trained).
\begin{definition}[Group-aware Bayes-optimal score (GABOS) functions]
We say that a score function $\scorefn$ is group-aware Bayes-optimal with respect to $\D$ if its output $\scorefn(x,a)$ is the empirical probability of individual $(x,a)\in\D$ being labelled positively; i.e., $\scorefn(x,a)=P(y=1 ~|~ x,a,\D)$.
\end{definition}

\begin{theorem}[Consistency of post-processing on GABOS functions] 
Suppose $\scorefn$ is a GABOS function with respect to $\D$, and that the loss $\lossfn$ measures mis-classification error of a thresholding of $\scorefn$ with respect to $\D$.  Furthermore, suppose the bias $\biasfn$ measures demographic parity or equal opportunity.  Then, the post-processing method (Algorithm~\ref{alg:post-processing}) applied to $\scorefn$ yields slack-consistent solutions.
\end{theorem}

\begin{figure}[h]
\centering
\includegraphics[width=0.66\columnwidth]{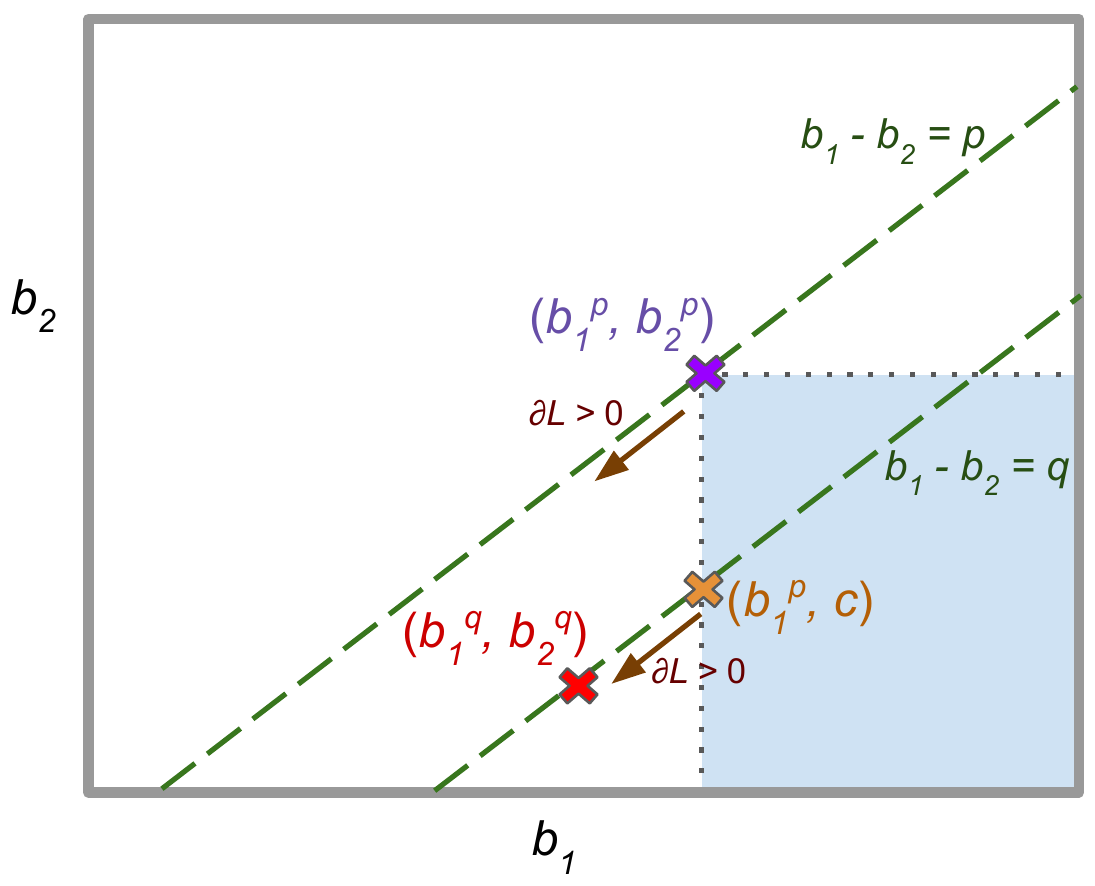}
\caption{A pictorial presentation of the proof of Claim 2 in Theorem 3.  To prove slack-consistency, we wish to show that $(b_1^q,b_2^q)$ lies within the shaded rectangle to the right and bottom of $(b_1^p,b_2^p)$.  Our proof is by way of contradiction.  If our claim does not hold, then we are able to show that $(b_1^p,c)$ is a solution with lower loss but same bias as $(b_1^q,b_2^q)$.}
  \label{fig:consistency-proof}
\end{figure}

\begin{proof}
The setting of the theorem allows us to split $\lossfn,\biasfn$ into functions of the two separate groups:
\begin{align}
    \lossfn(\tau_1,\tau_2) &= \lossfn_1(\tau_1) + \lossfn_2(\tau_2), \\ \biasfn(\tau_1,\tau_2) &= \biasfn_1(\tau_1) - \biasfn_2(\tau_2).
\end{align}
We note that in the considered setting, the functions $\lossfn_1,\lossfn_2$ are convex with respect to $\tau_1,\tau_2$ and the functions $\biasfn_1,\biasfn_2$ are monotonically decreasing.
We will find it useful to consider these functions in terms of the biases $b_1=\biasfn_1(\tau_1),b_2=\biasfn_2(\tau_2)$ induced by $\tau_1,\tau_2$.  That is, we express the loss and bias functions as,
\begin{align}
    \lossfn(b_1,b_2) &= \lossfn_1(b_1) + \lossfn_2(b_2), \\
    \biasfn(b_1,b_2) &= b_1 - b_2,
\end{align}
where we compute $\lossfn_i(b_i)$ as $\min_{\biasfn_i(\tau_i) = b_i}\lossfn_i(\tau_i)$.  We note that in the considered setting, the range of $b_1,b_2$ is $[0,1]$.  Furthermore, the functions $\lossfn_1,\lossfn_2$ maintain their convexity with respect to $b_1,b_2$ in the case of demographic parity and equal opportunity with a Bayes-optimal score function (see the appendix for a short proof).
We denote the left and right subdifferentials of $\lossfn_i$ by, 
\begin{align}
    \deriv \lossfn_i^{-} (b) &= \inf \deriv\lossfn_i(b), \\
    \deriv \lossfn_i^{+} (b) &= \sup \deriv\lossfn_i(b).
\end{align}

Consider a solution $\tau_1^p,\tau_2^p$ returned by post-processing with $\maxslack=p>0$.  Let $b_1^p = \biasfn(\tau_1^p),b_2^p=\biasfn(\tau_2^p)$. Without loss of generality, we may assume $b_1^p - b_2^p = p$; i.e., $b_1^p>b_2^p$.
By the optimality of $b_1^p,b_2^p$ we have,
\begin{align}
    \deriv\lossfn_1^{-}(b_1^p) &< 0, \hspace{0.2cm}
    \deriv\lossfn_2^{+}(b_2^p) \ge 0.
\end{align}
Otherwise, we would be able to achieve lower bias and lower loss simultaneously.  By the same logic we have,
\begin{align}
    \hspace{-0.3cm} \deriv\lossfn_1^{+}(b_1^p) + \deriv\lossfn_2^{+}(b_2^p) &\ge 0,~ \hspace{0.1cm}
    \deriv\lossfn_1^{-}(b_1^p) + \deriv\lossfn_2^{-}(b_2^p) < 0.
    \hspace{-0.1cm}
    \label{eq:grad-ineq}
\end{align}

Now consider an analogous solution $\tau_1^q,\tau_2^q$ for $\maxslack=q$ such that $q>p$ with bias values $|b_1^q-b_2^q| = q$.  We will show that $b_1^q \ge b_1^p$ and $b_2^q \le b_2^p$ through a sequence of three claims:

\textbf{Claim 1:} $b_1^q > b_2^q$.  \textit{Proof:}  Suppose otherwise; i.e., $b_2^q - b_1^q = q > 0$.  Then we must have at least one of $b_1^q<b_1^p$ or $b_2^q > b_2^p$ (otherwise we contradict $b_1^p>b_2^p$).  Suppose, first, that $b_1^q<b_1^p$.  By the convexity of $\lossfn_1$, we have,
\begin{equation}
    \deriv\lossfn_1^{+}(b_1^q) \le \deriv\lossfn_1^{-}(b_1^p) < 0.
\end{equation}
This means that we may increase $b_1^q$ to simultaneously lower the loss and bias; contradiction.  The same logic for the case of $b_2^q>b_2^p$ leads to an analogous contradiction.

\textbf{Claim 2:} $b_1^q \ge b_1^p$.  \textit{Proof:}  Suppose otherwise; i.e., $b_1^q < b_1^p$.  Let $c=\max\{b_2^q, b_1^p - q\}$ (see Figure~\ref{fig:consistency-proof} for a pictoral presentation).  Since $\lossfn$ is convex we have,
\begin{equation}
    \label{eq:term0}
    \lossfn(b_1^q,b_2^q) \ge \lossfn(b_1^p, c) + (b_1^q - b_1^p) \deriv\lossfn_1^{-}(b_1^p) + (b_2^q - c) \deriv\lossfn_2^{-}(c).
\end{equation}
Note that $b_2^q < b_1^q < b_1^p$ implies that $c = \max\{b_2^q, b_1^p - q\} < b_1^p$.  Combining this with the fact that $\deriv\lossfn_1^{-}(b_1^p) < 0$ we have,
\begin{equation}
    \label{eq:term1}
    (b_1^q - b_1^p) \deriv\lossfn_1^{-}(b_1^p) \ge (b_1^q - c) \deriv\lossfn_1^{-}(b_1^p).
\end{equation}
Furthermore we have,
\begin{equation}
    b_2^q = b_1^q - q < b_1^q - p = b_1^q - b_1^p + b_2^p < b_2^p,
\end{equation}
and,
\begin{equation}
    b_1^p - q < b_1^p - p = b_2^p,
\end{equation}
implying that $c = \max\{b_2^q, b_1^p - q\} < b_2^p$. Thus by the convexity of $\lossfn_2$, we have $\deriv\lossfn_2^{-}(c) \le \deriv\lossfn_2^{-}(b_2^p)$.  Recalling that $b_2^q -c \le 0$, we have,
\begin{equation}
    \label{eq:term2}
    (b_2^q - c) \deriv\lossfn_2^{-}(c) \ge (b_2^q - c) \deriv\lossfn_2^{-}(b_2^p).
\end{equation}
Combining equations~\ref{eq:term0}, \ref{eq:term1}, \ref{eq:term2} we have,
\begin{equation}
    \lossfn(b_1^q,b_2^q) \ge \lossfn(b_1^p, c) + (b_1^q - c) \deriv\lossfn_1^{-}(b_1^p) + (b_2^q - c) \deriv\lossfn_2^{-}(b_2^p).
\end{equation}
In conjunction with Equation~\ref{eq:grad-ineq}, this means that $(b_1^p, c)$ is a feasible solution for $\maxslack=q$ with lower loss; contradiction.

\textbf{Claim 3:} $b_2^q \le b_2^p$.  \textit{Proof:}  Analogous to the Claim 2.

These claims show that the biases of the optimal solution are monotonic in $\maxslack$.  Accordingly, the thresholds are monotonic as well, which implies that the solutions are slack-consistent, as desired.
\end{proof}

Given the previous theoretical result, we propose to learn a fair ML classifier by first learning a GABOS function and then applying post-processing.  This procedure is summarized in Algorithm~\ref{alg:gabos}.
Learning a suitable GABOS function may be performed in an unsupervised manner (e.g., using clustering) or in a supervised manner (e.g., using decision trees to minimize loss).  
We have the following result, which ensures that Algorithm~\ref{alg:gabos} yields slack-consistent solutions:

\begin{algorithm}[t]
   \caption{GABOS learning with post-processing.}
\begin{algorithmic}\label{alg:gabos}
   \STATE {\bf Inputs}: Dataset $\dataset=\{(x_i,a_i,y_i)\}_{i=1}^N$. Allowed slack $\maxslack\in\R_{>0}$.
   \STATE \textbullet~ Separate the dataset by group $\dataset_j = \{(x_i,a_i,y_i)~|~ a_i = j\}$ for $j=1,2$.  
   \STATE \textbullet~ For each $\dataset_j$, devise a finite partition $c_j$ of the feature space; e.g., $c_j(x) = m\in\{1,\dots,M\}$.  This can be done in an unsupervised fashion (e.g., clustering) or with supervision (e.g., a decision tree with respect to loss). Each partition must contain at least one member of $\D$.
   \STATE \textbullet~ For each group $j\in[1,2]$ and each cluster $m\in[1,M]$, compute the Bayes-optimal value for examples mapped to $m$; i.e., compute functions $f_j(m) := \frac{|\{i ~|~ y_i=1, x_i\in\D_j, c_j(x_i)=m\}|}{|\{i ~|~ x_i\in\D_j, c_j(x_i)=m\}|}$.
   \STATE \textbullet~ Run post-processing (Algorithm~\ref{alg:post-processing}) on the score function $\scorefn(x) = f_{\A(x)}(c_{A(x)}(x))$ with slack $\maxslack$.
   \STATE \textbullet~ Return $\PP(\scorefn(\cdot))$, where $\PP$ is the result of the optimal thresholding found by post-processing.
\end{algorithmic}
\end{algorithm}

\begin{theorem}[Consistency of Algorithm~\ref{alg:gabos}] 
Performing GABOS learning with post-processing yields slack-consistent solutions.
\end{theorem}
\vspace{-0.2in}
\begin{proof}
Algorithm~\ref{alg:gabos} essentially performs post-processing with respect to a GABOS function on a simplified dataset $\widetilde{\D} = \{(c_{a_i}(x_i), a_i, y_i)~|~ (x_i,a_i,y_i)\in \D \}$.  This yields slack-consistent solutions with respect to the reduced features $c_{a_i}(x_i)$.  Now consider some arbitrary $x$ (not necessarily in the original training set).  This individual will be mapped to a single partition $c_{\A(x)}(x)$ regardless of slack.  Therefore, its predictions $\PP(f_{\A(x)}(c_{\A(x)}(x)))$ will be monotonic with slack.
\end{proof}

Although we possess a solution to slack-consistency, we stress that this solution is not ideal for many practical scenarios for two reasons.  First, the use of stochastic thresholds is undesirable, since, as mentioned before, they may be seen as either capricious or exploitable.  Second, access to the group membership of an individual is often not available during inference (e.g., in web applications).
\section{DISCUSSION}
\vspace{-0.2cm}
Our work suggests that there is much to explore to make ML fairness methods more transparent. 
Conventional wisdom in machine learning suggests that lack of transparency and explainability arises from the use of complicated models.
However, our work shows that in the context of ML fairness, even with the simplest underlying models and the most straightforward training procedures, introducing a single fairness constraint can have significant consequences on the understandability of the model.
We may compare and contrast this to the phenomenon of adversarial examples in neural networks~\cite{goodfellow2014explaining}.
In the adversarial example setting, the complexity of the model leads to drastically unintended behavior. In our setting, it is the introduction of simple group-parity constraints which leads to counter-intuitive behavior.
We encourage researchers and practitioners to be wary of such complexity that may be introduced through seemingly simple augmentations to their models or training procedures. 

Our findings also uncover a stark disconnect between group-based fairness metrics and intuitive notions of fairness.
Previous works have noted the disconnects between group-parity and individual notions of fairness~\cite{dwork2012fairness} as well as between group-fair classifiers and future impact of decisions of those classifiers~\cite{liu2018delayed}.
Our work provides further evidence of this disconnect through the notion of slack-consistency.  Notably, our counter-examples show that standard methods for ML fairness violate slack-consistency for both individuals and groups as a whole on average, even when these individuals and groups come from the same data that the model was trained on.  Thus, we encourage researchers to re-assess the utility of using group-parity as a proxy for fairness.  

To conclude, we re-affirm that slack-consistency is a generally desirable behavior and can protect an ML model from a wide range of unreasonable behaviors. 
As argued previously, it is natural to expect that, for any individual, there would be a prediction under no slack (i.e. perfectly satisfy fairness constraint) and a prediction under infinite slack (i.e. unconstrained), and that for any slack in between, the predictions would change monotonically between these two extremes.
%if an individual from a disadvantaged group was given a positive prediction under no fairness constraints, then with constraints, slack-consistency would ensure that this prediction stays positive and 
This way, an individual would under no circumstances be unfairly treated for the benefit of the group.
Slack-consistency can also encourage predictions to be more robust: in practice, models have to be possibly frequently retrained, and slack consistency can ensure that small or even no changes in the fairness requirements would not lead to unreasonable changes in individual predictions. 
For these reasons, researchers and practitioners may find it beneficial to enforce slack-consistency itself in order to better guarantee a classifier's behavior.

\clearpage
{
\bibliography{ref}
\bibliographystyle{plain}
}
{\onecolumn
\newpage
\appendix
\section{Supporting Theoretical Results}
\begin{theorem}[Convexity of mis-classification error]
Considering all possible normalized thresholds $\tau\in[0,1]$ of a Bayes-optimal score function, the mis-classification error is a convex function of the true-positive rate (equal opportunity).
The mis-classification error is also convex with respect to positive prediction rate (demographic parity).
\end{theorem}

\begin{proof}
We first note that the ROC of a Bayes-optimal score function is a concave one-to-one function~\cite{sacchetto2018proper}. That is, the true positive rate $\tpr$ is a concave one-to-one function with respect to the false positive rate $\fpr$, and the false positive rate $\fpr$ is a convex one-to-one function with respect to the true positive rate $\tpr$.

The mis-classification error of a thresholding may be expressed as $\alpha_1\cdot(1-\tpr) + \alpha_2\cdot\fpr$, where $\alpha_1,\alpha_2>0$ are the proportions of positive and negative labelled points in the dataset, respectively.
Since the first term of this expression is linear and the second term convex with respect to $\tpr$, we conclude that the mis-classification error is a convex function of $\tpr$, as desired.

To characterize the mis-classification with respect to positive prediction rate $\pr$, we note that $\pr = \alpha_1\cdot \tpr + \alpha_2 \cdot \fpr$. Since the first term of this expression for $\pr$ is concave one-to-one and the second term linear increasing with respect to $\fpr$, we deduce that $\pr$ is a concave one-to-one function with respect to $\fpr$; equivalently, $\fpr$ is a convex one-to-one function with respect to $\pr$.
The mis-classification error may be expressed as $\alpha_1 - \pr + 2\alpha_2\cdot\fpr$, which is the sum of a constant, linear, and convex function with respect to $\pr$. Thus, we conclude that the mis-classification error is a convex function of $\pr$, as desired.
\end{proof}

\section{Additional details of Theorem~\ref{theorem:counterexample_bayesoptimal} and Figure~\ref{fig:discrete_equal_opp_post_processing}}
\begin{figure*}[h]
\begin{center}
\begin{tabular}{lll}
\includegraphics[width=0.32\textwidth]{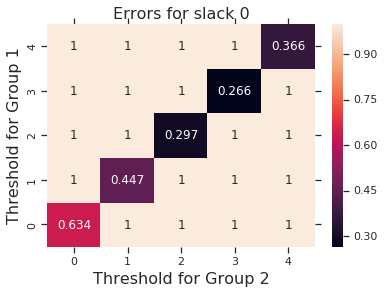} &
\includegraphics[width=0.32\textwidth]{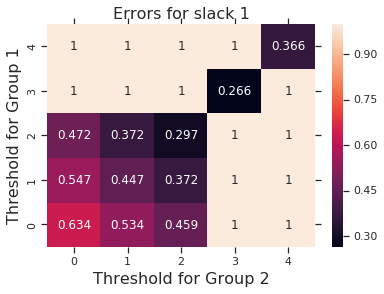}  & 
\includegraphics[width=0.32\textwidth]{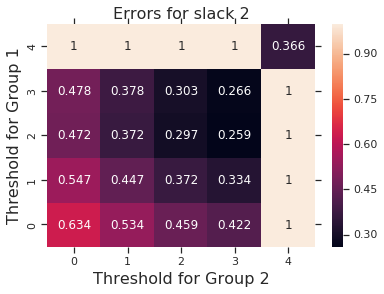}  \\
\includegraphics[width=0.32\textwidth]{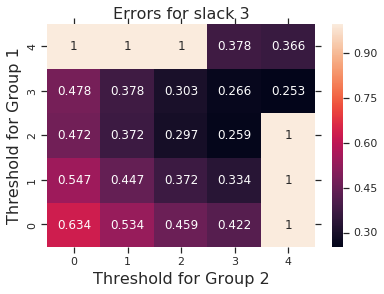} &
\includegraphics[width=0.32\textwidth]{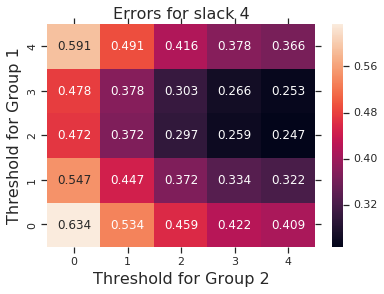}  & 
\includegraphics[width=0.32\textwidth]{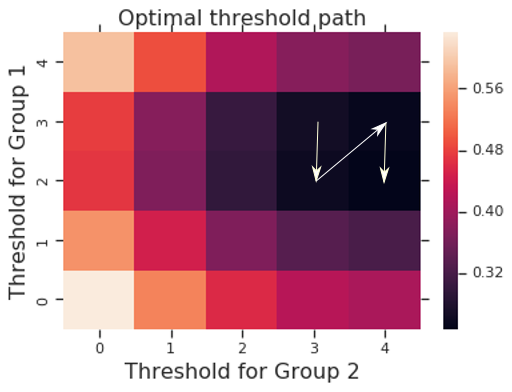}  
\end{tabular}
\end{center}
\caption{More detailed calculations for  Theorem~\ref{theorem:counterexample_bayesoptimal} and Figure~\ref{fig:discrete_equal_opp_post_processing}. 
The first $5$ charts show the error rate given pairs of thresholds for each of the groups for slacks $0, 1, 2, 3, 4$. If a pair was infeasible for a particular slack, then we display the error as $1$ for convenience. The last chart shows the path of the optimal pairs as slack increases showing that it is not monotonic in the thresholds.} 
  \label{fig:discrete_equal_opp_post_processing_all}
\end{figure*}
}
\end{document}